%% file: main.tex
\documentclass[11pt]{article}

\usepackage[margin=1in]{geometry}
\usepackage{amsmath,amssymb,amsthm,mathtools}
\usepackage{booktabs}
\usepackage{graphicx}
\usepackage{subcaption}
\usepackage{microtype}
\usepackage{natbib}
\usepackage{enumitem}
\setlist{topsep=2pt,itemsep=2pt,leftmargin=1.4em}
\usepackage{xcolor}
\usepackage{pgfplots}
\pgfplotsset{compat=1.17}
\usepackage[colorlinks=true,linkcolor=blue!55!black,citecolor=blue!55!black,urlcolor=blue!55!black]{hyperref}

\graphicspath{{figures/}}

\input{macros}

\title{\method{}: Saturation-Aware Mixed-Precision Quantization}

\author{Joshua Hill\thanks{Work done during an internship at Baseten Labs, Inc.}\\
Baseten Labs, Inc.\\
David R.\ Cheriton School of Computer Science, University of Waterloo}
\date{}

\begin{document}
\maketitle

\begin{abstract}
Mixed-precision quantization assigns each module of a language model its own precision format, such as NVFP4,
FP8, or BF16, to minimize damage within a memory budget. Most allocators predict the damage of a
configuration by summing per-unit sensitivities measured in unquantized networks. But the sensitivity of a unit \textit{saturates} 
as more units are quantized, so the
sum overprices the heavily quantized configurations that tight budgets require. Assuming only
that damage accumulates by an associative, increasing, continuous rule restricts it to the form
$g\bigl(\sum_i b_i\bigr)$, with summed sensitivities as the case of linear $g$. We propose \method{},
which uses a concave $g$, fits the per-unit parameters $b_i$ from a single Aumann--Shapley path
integral, and allocates formats with a linear program. On
Qwen3-30B-A3B, \method{} lowers KL divergence to BF16 by 11--23\% relative to ModelOpt AutoQuantize at
4.2--6.0 effective bits, and it outperforms an interaction-aware allocator that measures up to 2{,}410
configurations while measuring only three. It also predicts the damage of unmeasured configurations
with a median error of 6.9--11.4\%, lower than Gaussian-process and random-forest surrogates fit to
300 configurations, and served allocations of Qwen3.6-27B and GLM-5.2 (753B) run faster than vendor
NVFP4 checkpoints at comparable accuracy.
\end{abstract}

\input{sections/01_intro}
\input{sections/02_setup}
\input{sections/03_method}
\input{sections/04_experiments}

\input{sections/06_related}
\input{sections/07_limitations}

\input{sections/08_conclusion}

\bibliographystyle{plainnat}
\bibliography{references}

\clearpage
\appendix
\input{appendix/a_coverage_class}
\input{appendix/b_as_proofs}
\input{appendix/f_setup}
\input{appendix/g_tables}

\end{document}

%% file: macros.tex
\newcommand{\method}{SaturQ}

\theoremstyle{plain}

\newcommand{\E}{\mathbb{E}}

%% file: sections/01_intro.tex
\section{Introduction}\label{sec:intro}

Low-precision quantization reduce the memory and bandwidth cost of serving large language models.
Mixed-precision quantization therefore assigns each unit of the model, such as a
layer or a group of experts, its own quantization format, and keeps the units whose quantization is most harmful 
to model performance in
higher precisions.

Most allocators score each unit on its own and add the scores. HAWQ-v2 \citep{dong2020hawqv2} scores a
unit by the trace of its Hessian, ModelOpt AutoQuantize \citep{modelopt_autoquantize} by a
gradient-based estimate of the loss increase when the unit alone is quantized, and HIGGS
\citep{malinovskii2025higgs} by a per-layer coefficient on the squared reconstruction error, justified
by a theorem that the perplexity increase is linear in these errors at moderate bitwidths. For
mixture-of-experts models, MC-MoE \citep{mcmoe2025} scores experts by routing frequency and MoPEQ
\citep{mopeq2025} by Hessian trace. Each method predicts the damage of a configuration as the sum of
the scores of its quantized units and chooses the configuration of least predicted damage under the
budget, by a knapsack, dynamic program, or linear program.

However, modeling the quantization loss as a sum of per-layer terms isn't accurate in practice. 
Damage \textit{saturates}, meaning that
quantizing two units together tend to cost less than the sum of quantizing each alone. 
As a result, a sensitivity measured with every other unit at full
precision overstates the unit's cost relative to a configuration that quantizes most units, which is the
configuration a tight budget requires (Figure~\ref{fig:saturation}). The gap is large on real
models. On Qwen3-30B-A3B and Qwen3-32B, the sensitivities of 24 units measured in isolation sum to
5.9 and 3.5 times their Shapley values, which are their average marginal costs over all orders of
quantizing the units (Appendix~\ref{app:mle}).

Allocators that do not assume additivity instead optimize against the quantized model itself, at the
cost of many evaluations of it. EvoPress
\citep{sieberling2025evopress} observes that per-layer scores do not predict the loss of a compressed
model and searches configurations by evolution, measuring each candidate. CoopQ \citep{impq2025}
estimates Shapley values and pairwise interactions from hundreds to thousands of measured
configurations, and AMQ \citep{amq2025} fits a radial basis function surrogate to measured configurations.

We propose \method{}, an allocator that models this observed saturation. We assume only that damage
accumulates from per-unit terms by a fixed rule that is associative, increasing, and continuous.
These assumptions restrict the damage of a configuration $S$ to the form $g\bigl(\sum_{i\in S}
b_i\bigr)$ for an increasing function $g$, with summed sensitivities as the case of linear $g$
\citep{aczel1966fe,ling1965assoc}. We use the concave choice $g(u)=c\,(1-e^{-u})$, the \textit{coverage model},
which has one parameter per unit and a shared ceiling $c$. Rather than fit these parameters to many
measured configurations, we read them from a single Aumann--Shapley path integral
\citep{aumann1974values}. Under the coverage model, minimizing predicted damage within a
memory budget can be solved with a linear program.

On Qwen3-30B-A3B, \method{} lowers calibration KL divergence to BF16 by 11--23\% relative to ModelOpt
AutoQuantize at 4.2--6.0 effective bits. On both Qwen3-30B-A3B and Qwen3-235B-A22B it is significantly
better than CoopQ \citep{impq2025}, which measures the damage of 942--2{,}410 configurations, while
\method{} measures three. The coverage model also predicts the damage of unmeasured configurations
with a median error of 6.9\% and 11.4\% on the two models, against 13.8--19.2\% for a Gaussian process
and a random forest fit to 300 measured configurations, surrogates of the kind used for quantization search by AutoQRA
\citep{autoqra2026} and MiCo \citep{mico2025}. In serving, a \method{} allocation of
Qwen3.6-27B at 5.41 effective bits, fewer than either vendor NVFP4 checkpoint, stays within 2.5 points
of the better one on each benchmark with the highest throughput, and an allocation of GLM-5.2 (753B) decodes $1.19\times$
faster than NVIDIA's NVFP4 checkpoint with at least its accuracy.

Our contributions are the following.
\begin{enumerate}
  \item We show that sensitivity-based allocators assume additive damage, that measured damage
  saturates instead, and that assuming damage accumulates by an associative, increasing, continuous
  rule restricts it to the form $g\bigl(\sum_i b_i\bigr)$, which contains the coverage model
  (Section~\ref{sec:coverage}).
  \item We recover the parameters of the coverage model from one Aumann--Shapley path integral,
  without fitting to measured configurations, and allocate formats with a linear program
  (Sections~\ref{sec:as} and~\ref{sec:lp}).
  \item On models from 27B to 753B parameters, we show lower damage at matched memory than
  sensitivity-based and interaction-aware allocators, more accurate prediction of damage than
surrogates fit to 300 configurations, and served checkpoints that are faster than vendor NVFP4
checkpoints at comparable accuracy (Section~\ref{sec:experiments}).
\end{enumerate}

%% file: sections/02_setup.tex
\section{Problem setup}\label{sec:setup}

Let $U=\{1,\dots,L\}$ be the units of a model, where a unit is an atomic quantizable module (e.g. a single layer or a fused group such as a fused QKV module).
A configuration assigns each unit $i$ a format
$q(i)\in\mathcal{Q}$, for example $\mathcal{Q}=\{\text{NVFP4},\text{FP8},\text{BF16}\}$. For a
single low-precision format, a configuration is the set $S\subseteq U$ of quantized units, and its
\emph{damage} is the quantization loss
\begin{equation}
  f(S) = \mathbb{E}_x\bigl[\mathrm{KL}\bigl(p(\cdot\mid x)\,\|\,p_S(\cdot\mid x)\bigr)\bigr],
  \label{eq:damage}
\end{equation}
where $p$ and $p_S$ are the next-token distributions of the full-precision model and of the model
with the units in $S$ quantized. Memory is measured in
effective bits, $\sum_i N_i\,\mathrm{bits}(q(i))/\sum_i N_i$, where $N_i$ is the parameter count of
unit $i$. Mixed-precision quantization chooses the configuration of least damage within an
effective-bits budget $\bar b$.

Existing methods \citep{dong2020hawqv2,modelopt_autoquantize} predict damage by summing
sensitivities,
\begin{equation}
  \hat f(S) = \sum_{i\in S} s_i,
  \label{eq:sum}
\end{equation}
where $s_i$ is the damage of quantizing unit $i$ with every other unit at full precision, measured
directly or approximated from gradients or Hessians of the loss.

%% file: sections/03_method.tex
\section{\method{}}\label{sec:method}

We develop \method{} in three parts. Section~\ref{sec:coverage} introduces a model of quantization loss as a function of the subset of layers quantized,
 Section~\ref{sec:as} introduces a cheap gradient-based method to fit the model, and
Section~\ref{sec:lp} uses it to choose a quantization format for each unit.

\subsection{The coverage model}\label{sec:coverage}

To allocate precision we need to predict $f(S)$ for configurations we have not measured. Summing
per-unit sensitivities is one such prediction. Rather than choose a model directly, we start from
mild assumptions on how quantization loss should behave and derive a model.

Suppose quantization damage accumulates from per-layer terms according to a fixed rule,
\begin{equation}
  f(S\cup\{i\}) = f(S)\oplus d_i, \qquad i\notin S,
  \label{eq:accumulate}
\end{equation}
where $d_i$ depends only on unit $i$. We make four assumptions on the operation $\oplus$: associativity, strictly increasing in each argument,
continuity and has $0$ as an identity.

These assumptions fix the form of $f$ up to a single function. The Acz\'el--Ling representation
theorem \citep{aczel1966fe,ling1965assoc} states that any operation with these properties is
ordinary addition after a change of scale: there is a continuous, strictly increasing $g$ with
$x\oplus y = g\bigl(g^{-1}(x)+g^{-1}(y)\bigr)$. Applying this to Eq.~\eqref{eq:accumulate} gives
\begin{equation}
  f(S) = g\Bigl(\sum_{i\in S} b_i\Bigr), \qquad b_i = g^{-1}(d_i).
  \label{eq:gsum}
\end{equation}
Notice that summing per-layer sensitivities is
the case $g(x)=x$. Saturation occurs precisely when $g$ is concave, so that each additional unit of
cost adds less damage than the one before. We provide more details in Appendix~\ref{app:coverage-class},
and Appendix~\ref{app:structure} measures the saturation directly.

We use the coverage generator $g(u) = c\,(1-e^{-u})$, a simple choice that is bounded and concave. Substituting it into Eq.~\eqref{eq:gsum} and writing
$a_i = 1-e^{-b_i}$ gives
\begin{equation}
  f(S) = c\Bigl(1-\prod_{i\in S}(1-a_i)\Bigr),
  \label{eq:coverage}
\end{equation}
which is the noisy-OR, or probabilistic sum \citep{pearl1988probabilistic,klement2000tnorms}. 

Denoting $c-f(S)$ as the headroom left after quantizing $S$, 
quantizing an additional unit removes a fixed fraction of the remaining headroom,
\begin{equation}
  c - f(S\cup\{i\}) = (1-a_i)\bigl(c - f(S)\bigr).
  \label{eq:headroom}
\end{equation}
A model that is already badly damaged has little headroom left, so one more quantized unit barely increases the quantization loss.

This is why summed sensitivities overprice aggressive configurations. The marginal cost of unit
$i$ when the units in $S$ are already quantized is
\begin{equation}
  f(S\cup\{i\}) - f(S) = c\,a_i\prod_{j\in S}(1-a_j),
  \label{eq:marginal}
\end{equation}
which shrinks as $S$ grows. 
A realistic budget quantizes most
units, so the difference between the actual marginal increase in loss and the marginal increase predicted from per-layer sensitivities is large
(Figure~\ref{fig:saturation}). On real models the gap is several-fold: on Qwen3-30B-A3B and Qwen3-32B, the
sensitivities of 24 units sum to 5.9 and 3.5 times their Shapley values (Appendix~\ref{app:mle}).

\begin{figure}[t]
  \centering
  \input{figures/fig_saturation}
  \caption{Illustration with $L=48$ units of equal break-rate $a_i=0.05$. Both curves agree
  on every single unit, but the sum of per-unit sensitivities ignores the shrinking headroom
  and overprices the fully quantized model by roughly 160\%. The dashed line is the ceiling $c$.}
  \label{fig:saturation}
\end{figure}

\subsection{Recovering the parameters with Aumann--Shapley Integrals}\label{sec:as}

We can fit the $L+1$ parameters of the coverage model using gradient information.

Let $y_i$ be the output of unit $i$ on the calibration input, and let $\delta_i$ be the difference
between the unit's quantized and full-precision outputs. We move every unit's output
toward its quantized output at the same rate,
\begin{equation}
  y_i(t) = y_i + t\,\delta_i \quad\text{for all } i, \qquad t\in[0,1],
  \label{eq:path}
\end{equation}
so $t=0$ is the full-precision model and $t=1$ is the fully quantized one. Write $F(t)$ for the
damage at point $t$, so that $F(0)=0$ and $F(1)=f(U)$.

The Aumann--Shapley attribution of unit $i$ \citep{aumann1974values} is the gradient of $F$ with
respect to the unit's output, projected onto its quantization error and averaged along the path:
\begin{equation}
  \phi_i = \int_0^1 \bigl\langle \nabla_{y_i} F(t),\, \delta_i \bigr\rangle\,dt 
  \label{eq:as}
\end{equation}
It is the continuous analogue of the Shapley value, taken in $F$ rather than in $f$,
and relatedly, a similar path-integral construction is used in the problem of input-feature attribution \citep{sundararajan2017ig}.

We assume that $\phi_i$ approximates the Shapley value of unit $i$ in $f$, which is the integral along the
diagonal of the partial derivative of the multi-linear extension
$\tilde f(x_1,\ldots,x_n)=c\bigl(1-\prod_i(1-a_ix_i)\bigr)$ of $f$ \citep{owen1972multilinear}, an
assumption that holds empirically (Appendix~\ref{app:mle}), so that
\begin{equation}
  \phi_i = \int_{0}^1 \frac{\partial \tilde f}{\partial x_i}(t,\ldots,t)\,dt = c\,a_i \int_0^1 \prod_{j\ne i} (1-t\,a_j)\,dt
  \label{eq:as-coverage}
\end{equation}

By the chain rule, the above integrals should sum to $F(1)-F(0)=f(U)$, so we rescale the measured $\phi_i$ to
satisfy $\sum_i\phi_i=f(U)$. Given $\phi_i$ and the ceiling $c$
(Section~\ref{sec:lp}), Eq.~\eqref{eq:as-coverage} is a system of $L$ equations in the $L$ break-rates. We solve it by fixed-point iteration, starting from
$a_i = \phi_i/c$ and repeatedly setting
\[
  a_i \leftarrow \frac{\phi_i}{c\int_0^1 \prod_{j\ne i}(1-t\,a_j)\,dt}.
\]

resulting in our parameters.

\subsection{Allocation as a linear program}\label{sec:lp}

For each low-precision format $q\in\mathcal{Q}$, we run Section~\ref{sec:as},
moving every unit toward $q$, and obtain costs $b_{i,q} = -\log(1-a_{i,q})$ under a shared ceiling
$c$, with $b_{i,q}=0$ for BF16. The coverage model then predicts the loss of an assignment
$q(\cdot)$ of formats to units as
\begin{equation}
  \hat f = c\Bigl(1-\exp\Bigl(-\sum_i b_{i,q(i)}\Bigr)\Bigr).
  \label{eq:pred}
\end{equation}
Eq.~\eqref{eq:as-coverage} has a solution for every $c>f(U)$. We choose the $c$ at which Eq.~\eqref{eq:pred} reproduces the measured loss of one
fixed held-out configuration, which in our experiments quantizes the routed experts to NVFP4 and keeps every
other unit in BF16.
Since $\hat f$ is increasing in $\sum_i b_{i,q(i)}$, minimizing it under an effective-bits budget
$\bar b$ is an integer linear program over one-hot choices $x_{i,q}\in\{0,1\}$,
\begin{equation}
  \min_x \sum_{i,q} b_{i,q}\,x_{i,q}
  \quad\text{s.t.}\quad \sum_q x_{i,q} = 1 \;\;\forall i,
  \qquad \sum_{i,q} N_i\,\mathrm{bits}(q)\,x_{i,q} \le \bar b\sum_i N_i,
  \label{eq:lp}
\end{equation}
To find the smallest model with
$\hat f\le\varepsilon$ instead, we minimize the bit cost subject to
$\sum_{i,q} b_{i,q}\,x_{i,q} \le -\log(1-\varepsilon/c)$.

%% file: figures/fig_saturation.tex
\begin{tikzpicture}
  \begin{axis}[
      width=0.8\linewidth, height=5.2cm,
      xmin=0, xmax=48, ymin=0, ymax=2.6,
      xlabel={number of quantized units $|S|$},
      ylabel={damage $f(S)$ (units of $c$)},
      legend pos=north west, legend cell align=left,
      legend style={draw=none, fill=none, font=\small},
      tick label style={font=\small}, label style={font=\small},
    ]
    \addplot[gray, dashed, domain=0:48, forget plot] {1};
    \addplot[red!70!black, thick, domain=0:48, samples=49] {0.05*x};
    \addlegendentry{sum of per-unit sensitivities, $\sum_{i\in S} c\,a_i$}
    \addplot[blue!70!black, thick, domain=0:48, samples=49] {1 - 0.95^x};
    \addlegendentry{coverage, $c\bigl(1-\prod_{i\in S}(1-a_i)\bigr)$}
  \end{axis}
\end{tikzpicture}

%% file: sections/04_experiments.tex
\section{Experiments}\label{sec:experiments}

\subsection{Setup}\label{sec:expsetup}

We evaluate on Qwen3-30B-A3B and Qwen3-235B-A22B. We treat each attention projections,
routed experts and indiviidal linear layers as units, in order for the quantized checkpoint to use fused kernels. 
The quantization formats are
$\mathcal{Q}=\{\text{NVFP4},\text{FP8},\text{BF16}\}$ at 4, 8, and 16 bits.

The calibration prompts come from the BFCL and RULER benchmarks and are split
into two disjoint halves. 

We compare against four baselines. ModelOpt AutoQuantize \citep{modelopt_autoquantize} scores each
unit by a gradient-based sensitivity and allocates under the budget assuming the scores add. CoopQ
\citep{impq2025} estimates Shapley values from sampled permutations of the units and allocates with a
quadratic program. It evaluates the damage of 2{,}410 configurations on Qwen3-30B-A3B and 942 on
Qwen3-235B-A22B. MC-MoE \citep{mcmoe2025} ranks experts by routing frequency times parameter count,
and we report its form with one format per layer's experts, which fused kernels can serve. Uniform
gives every unit the lowest format that fits the budget, then returns the smallest units to BF16
until the budget is met. Appendix~\ref{app:tables} reports every method we measured, including
ablations of \method{}. 
We compare each method with \method{} by a paired $t$-test
over evaluation sequences, Holm-corrected within each budget at the 5\% level.

\subsection{Lower damage at matched memory}\label{sec:matched}

\input{generated/tab_klmain}

Table~\ref{tab:klmain} reports damage at matched effective bits. On Qwen3-30B-A3B, no baseline is
significantly better than \method{} at any budget. \method{}'s KL is 11--23\% lower than ModelOpt's
at 4.2--6.0 effective bits, and it is significantly lower than CoopQ's at
4.2--6.0, despite CoopQ evaluating the damage of 2{,}410 quantization configurations and \method{} evaluating only
$f(U)$.

On Qwen3-235B-A22B, \method{} is significantly better than CoopQ at 4.2--6.0 effective bits and than
ModelOpt at 4.2 and 6.0, with 17\% and 9\% lower KL against ModelOpt.

\subsection{Served checkpoints}\label{sec:served}

On Qwen3.6-27B, we compare this allocation at 5.41 effective bits against the two public
NVFP4 checkpoints, NVIDIA's at 6.31 effective bits and Unsloth's at 6.73, all served by vLLM on one
B200 at concurrency 128. Our checkpoint and Unsloth's quantize activations at runtime (W4A4), while
NVIDIA's keeps activations in BF16 (W4A16). Our accuracy is within 2.5 points of the better vendor
checkpoint on each benchmark (Table~\ref{tab:race27b}), and our throughput is the highest of the four
systems at each of the four prompt and generation lengths (Figure~\ref{fig:race27b}).

On GLM-5.2 (753B), we compare an allocation at 4.84 effective bits against
NVIDIA's NVFP4 checkpoint at 4.76, both W4A4 and served by SGLang on $4\times$B200 at
concurrency 8. Our accuracy is at least NVIDIA's on both benchmarks (Figure~\ref{fig:glm52}). Our
allocation reads 28.3~GB of weights per
decoded token against the checkpoint's 45.9~GB, and it decodes $1.19\times$ faster.

\begin{table}[t]
\centering\small
\caption{Qwen3.6-27B served accuracy, pass@1, with AIME25 averaged over 8 seeds.}
\label{tab:race27b}
\begin{tabular}{lcccc}
\toprule
 & eff.\ bits & AIME25 & GPQA & MMLU-Pro \\
\midrule
BF16 & 16 & 90.0 & 83.8 & 86.3 \\
NVIDIA NVFP4 & 6.31 & 90.8 & 67.2 & 85.8 \\
Unsloth NVFP4 & 6.73 & 90.4 & 82.3 & 86.7 \\
Ours & 5.41 & 90.0 & 79.8 & 86.5 \\
\bottomrule
\end{tabular}
\end{table}

\begin{figure}[t]
  \centering
  \includegraphics[width=\linewidth]{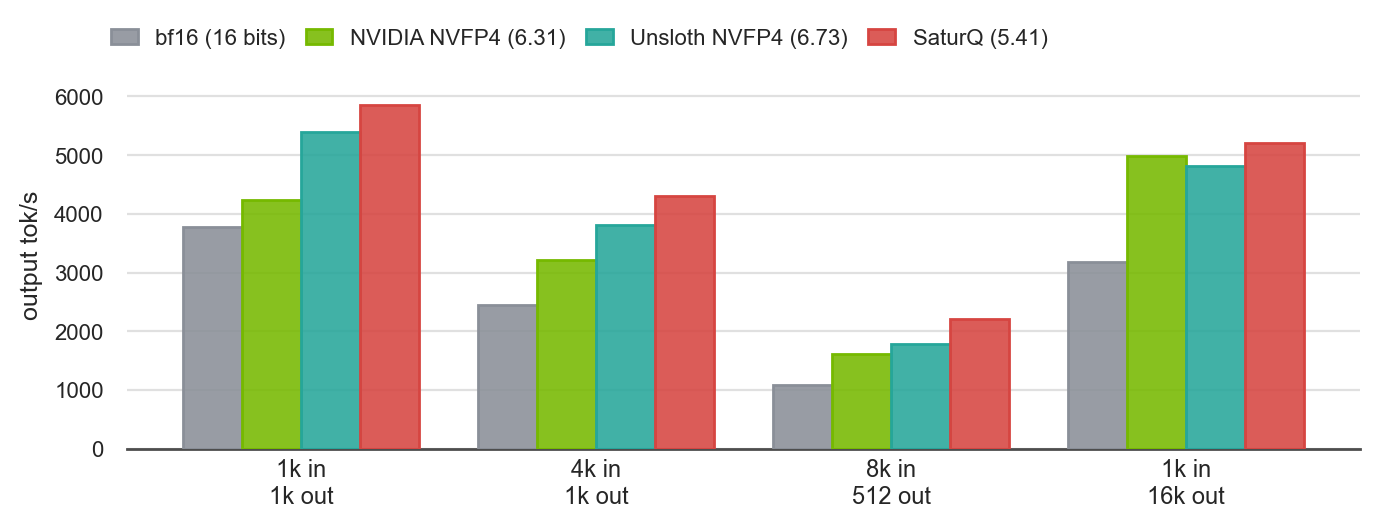}
  \caption{Qwen3.6-27B output throughput at four prompt and generation lengths, vLLM on one B200 at
  concurrency 128. Labels on our bars give the speedup over BF16.}
  \label{fig:race27b}
\end{figure}

\begin{figure}[t]
  \centering
  \includegraphics[width=0.5\linewidth]{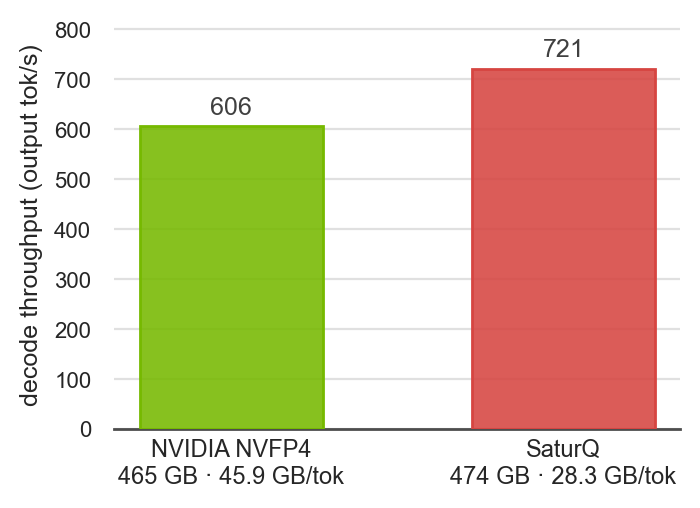}
  \caption{GLM-5.2 (753B) decode throughput, SGLang on $4\times$B200 at concurrency 8, with stored
  weights and weight bytes read per decoded token (batch size 1) below each bar.}
  \label{fig:glm52}
\end{figure}

\subsection{Predicting damage before deployment}\label{sec:pricing}

Eq.~\eqref{eq:pred} predicts the damage of a configuration before it is built. We test it on the
configurations that every method in Section~\ref{sec:matched} chose at every budget. 
As a baseline that uses the same attributions but ignores
saturation, we compute: $\hat f_{\mathrm{add}}(q)=\sum_i\phi_{i,q(i)}$. We also fit two surrogates used for quantization search, a Gaussian process with an RBF kernel
\citep{williams2006gpml}, after the Gaussian-process and RBF surrogates of AutoQRA \citep{autoqra2026}
and AMQ \citep{amq2025}, and a random forest, as in MiCo \citep{mico2025}, to the
measured damage of 300 configurations sampled from the first half.

\begin{table}[t]
\centering\small
\caption{Predicted against measured damage for the configurations of Section~\ref{sec:matched}.
Measured is the number of configurations whose damage the predictor needs. Error is the median of
$|\hat f-f|/f$, and Spearman is the rank correlation across all configurations.}
\label{tab:pricing}
\begin{tabular}{lccccc}
\toprule
 & & \multicolumn{2}{c}{Qwen3-30B-A3B} & \multicolumn{2}{c}{Qwen3-235B-A22B} \\
\cmidrule(lr){3-4}\cmidrule(lr){5-6}
Predictor & Measured & Error & Spearman & Error & Spearman \\
\midrule
\method{}, Eq.~\eqref{eq:pred} & 3 & \textbf{6.9\%} & \textbf{0.99} & \textbf{11.4\%} & \textbf{0.83} \\
Sum of $\phi_{i,q(i)}$ & 2 & 25.0\% & 0.98 & 31.0\% & 0.78 \\
Gaussian process & 300 & 14.7\% & 0.78 & 16.6\% & 0.67 \\
Random forest & 300 & 13.8\% & 0.60 & 19.2\% & 0.39 \\
\bottomrule
\end{tabular}
\end{table}

Table~\ref{tab:pricing} shows that \method{} predicts damage with a median error of 6.9\% and
11.4\%, the lowest of every predictor, from three measured configurations against 300 for the
Gaussian process and random forest. 

The sum of attributions underpredicts at every budget, by a median of
17--41\% on Qwen3-30B-A3B and 26--37\% on Qwen3-235B-A22B.

%% file: generated/tab_klmain.tex
\begin{table}[t]
\centering\small
\caption{Calibration KL to BF16 at matched effective bits. $^{+}$/$^{-}$: significantly worse/better than \method{} (paired $t$-test over 16 sequences, Holm-corrected within each row over all methods in Table~\ref{tab:kl30b}, which on Qwen3-235B-A22B are the same except int8-outlier). Bold: lowest in the row.}
\label{tab:klmain}
\resizebox{\textwidth}{!}{%
\begin{tabular}{cccccccccccc}
\toprule
 & \multicolumn{5}{c}{Qwen3-30B-A3B} & & \multicolumn{5}{c}{Qwen3-235B-A22B} \\
\cmidrule(lr){2-6}\cmidrule(lr){8-12}
eff & Uniform & ModelOpt & CoopQ & MC-MoE & \method{} & & Uniform & ModelOpt & CoopQ & MC-MoE & \method{} \\
\midrule
4.2 & .0514$^{+}$ & .0451$^{+}$ & .0451$^{+}$ & \textbf{.0397} & .0400 & & .0219 & .0247$^{+}$ & .0227$^{+}$ & \textbf{.0193} & .0205 \\
4.5 & .0374 & .0411$^{+}$ & .0424$^{+}$ & .0387 & \textbf{.0361} & & \textbf{.0182} & .0203 & .0225$^{+}$ & .0191 & .0186 \\
5.0 & .0375$^{+}$ & .0357$^{+}$ & .0379$^{+}$ & .0358$^{+}$ & \textbf{.0302} & & .0167 & .0174 & .0208$^{+}$ & .0182$^{+}$ & \textbf{.0163} \\
6.0 & .0350$^{+}$ & .0301$^{+}$ & .0328$^{+}$ & .0289$^{+}$ & \textbf{.0233} & & .0163$^{+}$ & .0145$^{+}$ & .0176$^{+}$ & .0146$^{+}$ & \textbf{.0132} \\
\bottomrule
\end{tabular}}
\end{table}

%% file: sections/06_related.tex
\section{Related work}\label{sec:related}

\paragraph{Sensitivity-based allocation.} Most mixed-precision methods score each unit on its own and
add the scores. HAWQ and HAWQ-v2 \citep{hawq2019,dong2020hawqv2} use the top Hessian eigenvalue or the
Hessian trace, ModelOpt AutoQuantize \citep{modelopt_autoquantize} uses a gradient-based score with a
linear program over formats, and HIGGS \citep{malinovskii2025higgs} weights each layer's squared
reconstruction error by a fitted coefficient and allocates by dynamic programming. FGMP \citep{fgmp2025} chooses between FP8 and NVFP4 for blocks within a
tensor using a Fisher-weighted score, which needs kernels that mix formats inside a tensor, and
ScaleBITS \citep{scalebits2026} allocates bitwidths for weights only. For mixture-of-experts models,
MC-MoE \citep{mcmoe2025}, MxMoE \citep{mxmoe2025}, GEMQ \citep{gemq2026}, and MoPEQ \citep{mopeq2025}
allocate precision across experts from importance scores such as routing frequency or Hessian trace. All of these predict the damage of a
configuration by summing per-unit scores, the estimator of Eq.~\eqref{eq:sum}, which overprices
aggressive configurations when damage saturates. Quantizers such as GPTQ \citep{frantar2022gptq} and
AWQ \citep{lin2024awq} decide how to quantize a unit, not which format it gets, and compose with any
allocator.

\paragraph{Interaction-aware allocation.} CoopQ \citep{impq2025} estimates Shapley values by
quantizing units in sampled orders, fits a model with linear and pairwise terms, and allocates with a
quadratic program. CLADO \citep{clado2023} measures pairwise interactions between layers by quantizing
them jointly. Both measure interactions directly, which takes many evaluations of the quantized model
(2{,}410 for CoopQ on Qwen3-30B-A3B in Section~\ref{sec:experiments}). EvoPress
\citep{sieberling2025evopress} shows that lower summed per-layer errors can give a worse compressed
model and searches configurations by evolution, measuring each candidate. \method{} captures the
interactions through a single concave function with one extra parameter, and reads its per-unit
parameters from one path integral.

\paragraph{Path attributions.} The Aumann--Shapley value \citep{aumann1974values} and integrated
gradients \citep{sundararajan2017ig} attribute a function's change to its inputs by integrating
gradients along a path, and Owen's multilinear extension \citep{owen1972multilinear} relates such
integrals to the Shapley value. PQI \citep{hu2025pqi} applies a path integral to the sensitivity of
individual weights but assumes their effects add. We apply it to the function that maps a set of
quantized units to damage, and use the result to fit the coverage model.

\paragraph{Surrogate predictors.} MiCo \citep{mico2025} predicts the accuracy of a quantization
configuration with a random forest, AMQ \citep{amq2025} with a radial basis function model, and AutoQRA
\citep{autoqra2026} with a Gaussian process. These need hundreds of measured configurations, and in
Section~\ref{sec:pricing} they predict the allocators' configurations less accurately than
\method{}.

%% file: sections/07_limitations.tex
\section{Limitations and future work}\label{sec:limitations}

\paragraph{Saturation is a property of activation quantization.} In earlier measurements on Qwen3-0.6B,
Qwen2.5-3B, and Qwen3-8B, we quantized the same units with weights only (W4A16) and with activations
only (W16A4). Predicting $f(S)$ by the sum of sensitivities measured in isolation, Eq.~\eqref{eq:sum},
gives a median relative error of 0.07--0.21 on held-out configurations for weight-only quantization,
against 0.35--2.83 for activation-only quantization. Weight-only damage is therefore close to
additive, so the coverage model offers little over summed sensitivities there.

\paragraph{Coupling under coarse activation scales.} Under Eq.~\eqref{eq:marginal}, the marginal
cost of a unit can only shrink as more units are quantized, so the coverage model cannot represent
units that cost more together than apart. 

\paragraph{Lower KL has not been shown to change served accuracy at 4 bits and above.} The KL
reductions of Section~\ref{sec:matched} measure closeness to the BF16 next-token distribution. We
have not shown that they produce accuracy gains on downstream benchmarks at 4.2--8.0 effective bits,
where the served models we evaluated are close to BF16 accuracy.

\paragraph{Effective bits as a proxy for inference speed.} The budget in Eq.~\eqref{eq:lp} optimizes
for the effective bit-width of the model, which could be a poor proxy for inference gains when the
forward-pass of the model is not memory-bound. An alternative inference speed proxy to optimize for is left
for future work.

%% file: sections/08_conclusion.tex
\section{Conclusion}\label{sec:conclusion}

Mixed-precision allocators predict the damage of a configuration by summing sensitivities measured in
isolation, which assumes that quantization damage adds up across units. Damage instead saturates, so
the sum overprices the aggressive configurations that tight memory budgets require. Assuming that damage accumulates by a fixed associative,
increasing, and continuous rule restricts it to the form $g\bigl(\sum_i b_i\bigr)$. 
\method{} uses the concave coverage model within this class, reads
its per-unit parameters from a single Aumann--Shapley path integral, and allocates formats with a
linear program. On Qwen3-30B-A3B,
\method{} lowers calibration KL by 11--23\% relative to ModelOpt AutoQuantize at 4.2--6.0 effective
bits, and it is significantly better than CoopQ at those budgets on both Qwen3-30B-A3B and
Qwen3-235B-A22B while measuring the damage of three configurations rather than the 942--2{,}410 that CoopQ
measures. The
same model predicts the damage of unmeasured configurations with a median error of 6.9\% and 11.4\%
from three measured configurations, against 13.8--19.2\% for surrogates fit to 300. Served, a
\method{} allocation of Qwen3.6-27B at 5.41 effective bits stays within 2.5 points of the better vendor
NVFP4 checkpoint on each benchmark at the highest throughput, and on GLM-5.2 an allocation at 4.84 effective
bits decodes $1.19\times$ faster than NVIDIA's NVFP4 checkpoint at 4.76, with at least its accuracy.

%% file: appendix/a_coverage_class.tex
\section{The coverage model class}\label{app:coverage-class}

This appendix states the representation result behind Eq.~\eqref{eq:gsum}, shows that concavity of
$g$ is equivalent to saturation, and reports a direct measurement of saturation.

\paragraph{The representation.} Let $\oplus$ be a binary operation on an interval $[0,M)$, with
$M\le\infty$, that is associative, continuous, strictly increasing in each argument, and has $0$ as
identity. The Acz\'el--Ling theorem \citep{aczel1966fe,ling1965assoc} states that there is a
continuous, strictly increasing $g:[0,\infty)\to[0,M)$ with $g(0)=0$ such that
\begin{equation*}
  x\oplus y = g\bigl(g^{-1}(x)+g^{-1}(y)\bigr),
\end{equation*}
and that $g$ is unique up to rescaling its argument by a positive constant. Suppose damage
accumulates by Eq.~\eqref{eq:accumulate}, $f(S\cup\{i\})=f(S)\oplus d_i$ for $i\notin S$, with
$f(\emptyset)=0$. Writing $b_i=g^{-1}(d_i)$ and adding the units of $S$ one at a time gives
$g^{-1}\bigl(f(S)\bigr)=\sum_{i\in S}b_i$ by induction on $|S|$, which is Eq.~\eqref{eq:gsum}.
The identity requirement
gives $f(\emptyset)=g(0)=0$. 

\paragraph{Concavity and saturation.} Write $B_S=\sum_{j\in S}b_j$. The marginal damage of unit
$i\notin S$ is $f(S\cup\{i\})-f(S)=g(B_S+b_i)-g(B_S)$. If $g$ is concave, this difference is
non-increasing in $B_S$, so every unit costs at most as much in a larger configuration as in a smaller
one, and for $S=\emptyset$ the marginal is the sensitivity measured in isolation, $f(\{i\})=g(b_i)$.
Concave $g$ therefore implies $f(S)\le\sum_{i\in S}f(\{i\})$, which is the overpricing of
Eq.~\eqref{eq:sum}. Linear $g$ gives equality.

\paragraph{The coverage generator.} For $g(u)=c\,(1-e^{-u})$, the operation is
$x\oplus y=x+y-xy/c$ on $[0,c)$, the probabilistic sum rescaled to the ceiling $c$
\citep{pearl1988probabilistic,klement2000tnorms}, and $g^{-1}(x)=-\log(1-x/c)$ is the logarithm of
the fraction of headroom used. Substituting $a_i=1-e^{-b_i}$ gives Eq.~\eqref{eq:coverage}. The
generator is bounded by $c$ and strictly concave, and its marginal damage, Eq.~\eqref{eq:marginal},
falls by the factor $1-a_j$ each time a unit $j$ is added to $S$.

\subsection{Measured saturation}\label{app:structure}

The class $g\bigl(\sum_i b_i\bigr)$ with concave $g$ predicts that the marginal damage of a unit falls
as more units are quantized. We test this prediction directly in earlier experiments on four models.
For a block of 8 consecutive layers at the middle of each network, we measure the damage of all
$2^8=256$ configurations of quantized layers, with damage defined as the increase in next-token
cross-entropy over the full-precision model. Table~\ref{tab:marginals} reports $f(S\cup\{i\})-f(S)$ over all
pairs of a layer $i$ and a configuration $S\not\ni i$ with $|S|$ layers already quantized.

The marginal damage at $|S|=0$, which is the sensitivity measured in
isolation, exceeds the mean marginal at every $|S|\ge 1$ on all four models, by a factor of 2.2--8.3
at $|S|=1$ on Qwen3-0.6B, Qwen3-8B, and Qwen3-30B-A3B and of 1.2 on Llama-3.2-3B.

\begin{table}[ht]
  \centering\small
  \caption{Mean marginal damage $f(S\cup\{i\})-f(S)$ by the number $|S|$ of layers already quantized,
  over every layer and configuration of an exhaustively measured 8-layer block at mid-depth. Damage is
  the increase in next-token cross-entropy, with every layer quantized to NVFP4 (W4A4).}
  \label{tab:marginals}
  \begin{tabular}{lcccccccc}
    \toprule
    & \multicolumn{8}{c}{$|S|$} \\
    \cmidrule(lr){2-9}
    Model & 0 & 1 & 2 & 3 & 4 & 5 & 6 & 7 \\
    \midrule
    Qwen3-0.6B & .183 & .022 & .024 & .022 & .021 & .023 & .022 & .012 \\
    Qwen3-8B & .023 & .006 & .006 & .006 & .006 & .007 & .003 & .009 \\
    Llama-3.2-3B & .016 & .013 & .010 & .009 & .009 & .007 & .009 & .003 \\
    Qwen3-30B-A3B & .060 & .027 & .024 & .023 & .022 & .016 & .021 & .021 \\
    \bottomrule
  \end{tabular}
\end{table}

%% file: appendix/b_as_proofs.tex
\section{The path integral and the Shapley value}\label{app:mle}

Section~\ref{sec:as} assumes that the path integral $\phi_i$ of Eq.~\eqref{eq:as} approximates the
Shapley value of unit $i$ in $f$,
\begin{equation*}
  \mathrm{Sh}_i(f)=\int_0^1\E_{S}\bigl[f(S\cup\{i\})-f(S\setminus\{i\})\bigr]\,dt,
\end{equation*}
where $S$ contains each unit independently with probability $t$ \citep{owen1972multilinear}.

We test this approximation on Qwen3-30B-A3B
and Qwen3-32B, with every unit quantized to NVFP4 (W4A4). Damage is the KL of
Eq.~\eqref{eq:damage} on 8 WikiText-2 sequences of 1{,}024 tokens. We compute $\phi_i$ by 8-point
Gauss--Legendre quadrature in $t$, obtaining each path derivative $\langle\nabla_{y_i}F(t),\delta_i\rangle$
by backpropagation through forward hooks that set every unit's output to Eq.~\eqref{eq:path}. We
estimate $\mathrm{Sh}_i$ for 24 units evenly spaced in depth by averaging
$f(S\cup\{i\})-f(S\setminus\{i\})$ over draws of $(t,S)$, with $t$ stratified over $[0,1]$ and each $S$
shared across the 24 units, using 128 draws for each model. As a baseline,
we compare $\mathrm{Sh}_i$ with the sensitivity measured in isolation, $f(\{i\})$.

Table~\ref{tab:shapley} reports the agreement across the 24 units. The path integral ranks the units with Spearman correlation 0.89 and 0.76 against the Shapley
values, while the sensitivities measured in isolation reach 0.15 and 0.63 and overstate their sum by
a factor of 5.9 and 3.5. The estimates of $\mathrm{Sh}_i$ are noisy, with a median standard error of
28\% (Qwen3-30B-A3B) and 75\% (Qwen3-32B) of the estimate, which lowers every correlation in the table.

\begin{table}[h]
\centering\small
\caption{Agreement of the path integral $\phi_i$ and of the sensitivity measured in isolation
$f(\{i\})$ with the sampled Shapley value $\mathrm{Sh}_i(f)$, over 24 units. The ratio of sums is the
sum of the predictor over the sum of the estimated $\mathrm{Sh}_i$.}
\label{tab:shapley}
\begin{tabular}{llccc}
\toprule
Model & Predictor & Pearson & Spearman & ratio of sums \\
\midrule
Qwen3-30B-A3B & $\phi_i$ & 0.98 & 0.89 & 0.77 \\
 & $f(\{i\})$ & 0.26 & 0.15 & 5.87 \\
\midrule
Qwen3-32B & $\phi_i$ & 0.98 & 0.76 & 1.10 \\
 & $f(\{i\})$ & 0.97 & 0.63 & 3.50 \\
\bottomrule
\end{tabular}
\end{table}

%% file: appendix/f_setup.tex
\section{Experimental setup}\label{app:setup}

\paragraph{Units and formats.} A unit is the smallest group of weights that a fused serving kernel
requires to share one format, namely the fused query, key, and value projection, the attention output
projection, the routed experts of one layer, and each remaining linear layer.

\paragraph{Calibration and evaluation data.} The calibration prompts are tool-use prompts from BFCL-v3
(the live\_simple, live\_multiple, simple\_python, and multiple splits) and validation prompts from
RULER at a 4{,}096-token context. We also select a disjoint half of prompts from these sources for our evaluations. Every quantity that informs an allocation is computed on the first half. This
includes the path attributions, the damage $f_q(U)$ of each fully quantized model, the configuration
that sets the ceiling $c$, CoopQ's permutation measurements, and the 300
configurations used to fit the Gaussian process and random forest. The damage reported in
Tables~\ref{tab:klmain} and~\ref{tab:pricing} is the KL of
Eq.~\eqref{eq:damage} on the second half.

\paragraph{\method{}.} For each low-precision format $q\in\{\text{NVFP4},\text{FP8}\}$, we compute
the attributions $\phi_{i,q}$ of Eq.~\eqref{eq:as} by Gauss--Legendre quadrature in $t$ with 8 nodes on
Qwen3-30B-A3B and 6 on Qwen3-235B-A22B, obtaining each path derivative by backpropagation through
forward hooks that set every unit's output to Eq.~\eqref{eq:path}. The router of each mixture-of-experts
layer is held at its full-precision expert assignment along the path, since the assignment is not
differentiable. Negative attributions are set to zero, because the coverage model cannot represent a
unit whose quantization reduces damage, and the remaining attributions are rescaled to sum to the
measured $f_q(U)$. We solve Eq.~\eqref{eq:as-coverage} for the break-rates by the fixed-point
iteration of Section~\ref{sec:as}, evaluating the integral by 32-point Gauss--Legendre quadrature, and iterate until no break-rate changes by
more than $10^{-12}$. The ceiling $c$ is the value at which Eq.~\eqref{eq:pred} reproduces the
measured damage of the configuration with the routed experts in NVFP4 and every other unit in BF16.

\paragraph{Baselines.} CoopQ \citep{impq2025} moves the units from FP8 to NVFP4 one at a time in random
orders and records each unit's marginal damage, with damage measured as the KL on the first
calibration half, with 10 orders on Qwen3-30B-A3B and 2 on
Qwen3-235B-A22B, for 2{,}410 and 942 measured configurations. From these it estimates per-unit values
and a pairwise interaction matrix shrunk halfway toward its diagonal, as in the original method.
MC-MoE \citep{mcmoe2025} scores each expert by the product of its routing frequency and its share of
routing weight on the calibration half, sums the scores of a layer's experts to score the layer's
routed-experts unit, weights the Frobenius reconstruction error of each format by this score, and
allocates with the same optimizer as \method{}, keeping attention in BF16.

\paragraph{Significance tests.} All methods at a given budget are evaluated on the same evaluation
sequences, so we compare each method with \method{} by a paired $t$-test over the per-sequence KL and
correct the tests within each row of Table~\ref{tab:klmain} by Holm's method at the 5\% level, over
every method we measured at that budget (Table~\ref{tab:kl30b} and Appendix~\ref{app:tables}).

\paragraph{Serving.} The Qwen3.6-27B checkpoints of Section~\ref{sec:served} are exported with
ModelOpt in W4A4 form and served by vLLM on one B200 at concurrency 128, and the GLM-5.2 checkpoints
by SGLang on four B200s at concurrency 8. Reasoning benchmarks are decoded at temperature 0.6 and
top-$p$ 0.95 with a budget of 81{,}920 generated tokens on AIME-2025 and 32{,}768 on GPQA diamond and
MMLU-Pro, and AIME-2025 accuracy is averaged over 8 seeds.

%% file: appendix/g_tables.tex
\section{Full matched-memory results}\label{app:tables}

Table~\ref{tab:kl30b} reports every method we measured on Qwen3-30B-A3B at each budget, which is the
set over which the significance tests of Table~\ref{tab:klmain} are Holm-corrected. On
Qwen3-235B-A22B we measured the same methods except int8-outlier. Beyond the
baselines of Section~\ref{sec:expsetup}, the columns are as follows. Naive-routed places the routed
experts in NVFP4 and keeps every other unit in BF16. Int8-outlier ranks units by an
LLM.int8()-style activation-outlier score \citep{dettmers2022llmint8}. Fit-add fits an additive model
$\hat f(S)=\sum_{i\in S}w_i$ by least squares to the measured damage of 300 configurations, and GP and
RF convert the Gaussian process and random forest of Section~\ref{sec:pricing}, fit to the same 300
configurations, into allocators by greedily choosing the demotions of least predicted damage. The
three \method{} variants change one component each. LS fit replaces the path integral with a
least-squares fit of the coverage model to the same 300 configurations, add.\ costs allocates with
the path attributions $\phi_{i,q}$ as additive costs and no coverage model, and greedy replaces the
integer program of Eq.~\eqref{eq:lp} with greedy demotion.

\input{generated/tab_kl30b}

%% file: generated/tab_kl30b.tex
\begin{table}[t]
\centering
\caption{Qwen3-30B-A3B, calibration KL to BF16 at matched effective bits, all methods measured on the same 16 held-out sequences. $^{+}$/$^{-}$: significantly worse/better than \method{} (paired $t$-test, Holm-corrected within each row at the 5\% level). Bold: lowest in the row. $^\dagger$MC-MoE's per-expert form gives each expert its own format, which fused MoE kernels cannot serve; it is a single measurement with no significance test.}
\label{tab:kl30b}
\resizebox{\textwidth}{!}{%
\begin{tabular}{ccccccccccccccc}
\toprule
eff & naive-unif & naive-routed & ModelOpt & int8-outlier & CoopQ (2{,}410 ev) & MC-MoE per-exp$^\dagger$ & MC-MoE fusion & fit-add (300 ev) & GP (300 ev) & RF (300 ev) & \method{}, LS fit (300 ev) & \method{}, add.\ costs & \method{}, greedy & \method{} \\
\midrule
4.2 & .0514$^{+}$ & .0412 & .0451$^{+}$ & .0451$^{+}$ & .0451$^{+}$ & .0397 & .0397 & .0430 & .0417 & .0599$^{+}$ & .0422 & .0393 & \textbf{.0391} & .0400 \\
4.5 & .0374 & .0374 & .0411$^{+}$ & .0377 & .0424$^{+}$ & .0359 & .0387 & .0362 & .0363 & .0532$^{+}$ & .0363 & .0362 & \textbf{.0356} & .0361 \\
5.0 & .0375$^{+}$ & .0374$^{+}$ & .0357$^{+}$ & .0337$^{+}$ & .0379$^{+}$ & .0322 & .0358$^{+}$ & .0326 & .0334$^{+}$ & .0462$^{+}$ & .0307 & .0314 & .0310 & \textbf{.0302} \\
6.0 & .0350$^{+}$ & .0374$^{+}$ & .0301$^{+}$ & .0296$^{+}$ & .0328$^{+}$ & .0286 & .0289$^{+}$ & .0261$^{+}$ & .0265$^{+}$ & .0396$^{+}$ & .0233 & \textbf{.0229} & .0232 & .0233 \\
8.0 & .0142 & .0374$^{+}$ & .0246$^{+}$ & .0205$^{+}$ & .0142 & .0230 & .0135 & .0191$^{+}$ & .0209$^{+}$ & .0348$^{+}$ & .0168$^{+}$ & \textbf{.0118} & .0122 & .0127 \\
\bottomrule
\end{tabular}}
\end{table}